\documentclass[10pt]{article}
\usepackage{natbib}
\usepackage{latexsym,amsmath,amssymb}
\usepackage{tikz}
\usepackage{algorithm}
\usepackage{algorithmic}

\usetikzlibrary{trees,arrows,automata}

\newcommand{\merl}{{MERL}}
\newcommand{\pac}{{PAC}}
\newcommand{\aO}[1]{\tilde O\left({#1} \right)}
\renewcommand{\H}{\mathcal H}

\newcommand{\unu}{\underline\nu}
\newcommand{\bnu}{\overline\nu}

\newcommand{\oU}{\overline U}
\newcommand{\uU}{\underline U}

\newcommand{\ceil}[1]{\left \lceil {#1} \right\rceil}
\renewcommand{\P}[1]{\operatorname{P}\left\{#1\right\}}
\newcommand{\argmax}{\operatornamewithlimits{arg\,max}}

\newcommand{\ind}[1]{[\![ #1 ]\!]}
\newcommand{\bigind}[1]{\left[\!\!\left[ #1 \right]\!\!\right]}

\newcommand{\prefix}{\sqsubset}

\newcommand{\E}{\mathbf E}
\newcommand{\K}{\mathcal K}
\newcommand{\F}{\mathcal F}
\newcommand{\M}{{\mathcal M}}

\newcommand{\R}[0]{\mathbb R}

\newcommand{\N}[0]{\mathbb N}

\newcommand{\constemaxkappa}{{ 2^{11}N \over \epsilon_\kappa^2(1 - \gamma)^2} \log^2{2^{9}N \over \epsilon^2(1 - \gamma)^2 \delta_1}}
\newcommand{\constgmaxkappa}{{ 2^{17}N  |\K| \over \epsilon \epsilon_\kappa(1 - \gamma)^2} \log^2{2^{9}N \over \epsilon^2(1 - \gamma)^2 \delta_1}}
\newcommand{\constemax}{{ 2^{16}N \over \epsilon^2(1 - \gamma)^2} \log^2{2^{9}N \over \epsilon^2(1 - \gamma)^2 \delta_1}}
\newcommand{\constgmax}{{ 2^{16}N |\K| \over \epsilon^2(1 - \gamma)^2} \log^2{2^{9}N \over \epsilon^2(1 - \gamma)^2 \delta_1}}
\newcommand{\constd}{{1 \over 1 - \gamma} \log{8 \over (1 - \gamma)\epsilon}}
\newcommand{\constalpha}{{4\sqrt{N} \over 4\sqrt{N} - 1}}
\newcommand{\constdeltaone}{{\delta \over 32 |\K| N^{3/2} }}
\newcommand{\constK}{{\log_2 {1 \over \epsilon(1 - \gamma)} + 2}}

\newcommand{\emaxkappa}{E_{\max,\kappa}}
\newcommand{\gmax}{G_{\max}}
\newcommand{\gmaxkappa}{G_{\max,\kappa}}
\newcommand{\emax}{E_{\max}}

\newcommand{\eqr}[1]{Equation (\ref{#1})}
\newcommand{\eqn}[1]{\begin{align}#1\end{align}}
\newcommand{\eq}[1]{\begin{align*}#1\end{align*}}

\usepackage{amsthm}
\theoremstyle{plain}
\newtheorem{theorem}{Theorem}
\newtheorem{lemma}[theorem]{Lemma}

\theoremstyle{definition}

\theoremstyle{remark}

\renewcommand{\qedsymbol}{$\blacksquare$}

\newenvironment{proofof}[1]{\par\vspace{1mm}\noindent{\bfseries\upshape Proof#1.}}{\hfill\qedsymbol \par\vspace{2mm}\noindent\ignorespacesafterend}
\renewenvironment{proof}{\par\vspace{1mm}\noindent{\bfseries\upshape Proof.}}{\hfill\qedsymbol \par\vspace{2mm}\noindent\ignorespacesafterend}

\newenvironment{keywords}{\centerline{\bf\small
Keywords}\begin{quote}\small}{\par\end{quote}\vskip 1ex}

\def\subsubsect#1{\vspace{1ex plus 0.5ex minus 0.5ex}\noindent{\bf\boldmath{#1.}}}

\begin{document}

\title{
\vskip 2mm\bf\Large\hrule height5pt \vskip 4mm
The Sample-Complexity of General Reinforcement Learning 
\vskip 4mm \hrule height2pt}
\author{{\bf Tor Lattimore} and {\bf Marcus Hutter} and {\bf Peter Sunehag} \\[3mm]
\normalsize Research School of Computer Science \\[-0.5ex] 
\normalsize Australian National University \\[-0.5ex]
\normalsize\texttt{\{tor.lattimore,marcus.hutter,peter.sunehag\}@anu.edu.au}
}
\date{July 2013}

\maketitle

\begin{abstract}
We present a new algorithm for general 
reinforcement learning where the true environment is known to belong to a finite class of $N$ arbitrary models.
The algorithm is shown to be near-optimal for all but $O(N \log^2 N)$ time-steps with high probability.
Infinite classes are also considered where we show that compactness is a key criterion for determining the
existence of uniform sample-complexity bounds. A matching lower bound is given for the finite case.

\def\contentsname{\centering\normalsize Contents}
{\parskip=-2.7ex\tableofcontents}

\end{abstract}


\newpage

\section{Introduction}

Reinforcement Learning (RL) is the task of learning policies that lead to nearly-optimal rewards where the environment is unknown.
One metric of the efficiency of an RL algorithm is sample-complexity, which is a high probability upper bound on the number of
time-steps when that algorithm is not nearly-optimal
that holds for all environment in some class. Such bounds are typically shown for very specific classes of environments,
such as (partially observable/factored) Markov Decision Processes (MDP) and bandits. We consider more general classes
of environments where at each time-step an agent takes an action $a \in A$ where-upon it receives reward $r \in [0,1]$ and
an observation $o \in O$, which are generated stochastically by the environment and may depend arbitrarily on the entire history sequence.
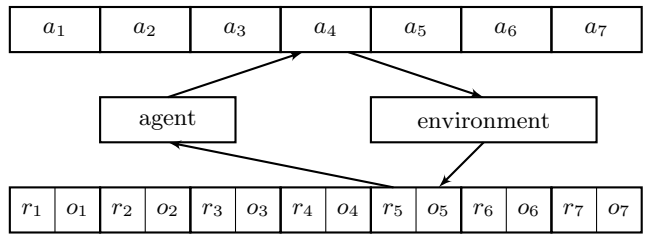
\begin{figure}[H]\label{fig_env}
\centering
\small
\begin{tikzpicture}[>=latex',scale=0.6,thick]
\tikzstyle{empty}=[draw=none]
\foreach \x in {1,2,3,4,5,6,7} {
  \draw (2*\x,1) rectangle (2*\x+2,2);
  \draw (2*\x,5) rectangle (2*\x+2,6);
  \draw[thin] (2*\x+1,1) -- (2*\x+1,2);

  \node[empty] at (2*\x+1,5.5) {$a_\x$};
  \node[empty] at (2*\x+0.5,1.5) {$r_\x$};
  \node[empty] at (2*\x+1.5,1.5) {$o_\x$};
}

\draw (4,3) rectangle (7, 4);
\draw (10,3) rectangle (15, 4);

\node[empty] at (5.5,3.5) {agent};
\node[empty] at (12.5,3.5) {environment};

\draw[->] (5.5,4) -- (8.5,5);
\draw[->] (9.5,5) -- (12.5,4);

\draw[->] (12.5,3) -- (11.5,2);
\draw[->] (10.5,2) -- (5.5,3);

\end{tikzpicture}
\caption{Agent/Environment Interaction}
\end{figure}

We present a new reinforcement learning algorithm, named Maximum Exploration Reinforcement Learning (\merl{}), that accepts as input a finite set 
$\M := \left\{\nu_1,\cdots, \nu_N\right\}$ of arbitrary environments, an accuracy $\epsilon$, and a confidence $\delta$. 
The main result is that \merl{} has a sample-complexity of
\eq{
\aO{{N \over \epsilon^2(1 - \gamma)^3} \log^2{N \over \delta \epsilon (1 - \gamma)}},
}
where $1/(1 - \gamma)$ is the effective horizon determined by discount rate $\gamma$.
We also consider the case where $\M$ is infinite, but compact with respect to a particular topology. In this case, a variant of \merl{} has
the same sample-complexity as above, but where $N$ is replaced by the size of the smallest $\epsilon$-cover.
A lower bound is also given that matches the upper bound except for logarithmic factors. Finally, if $\M$ is
non-compact then in general no finite sample-complexity bound exists.

\subsubsect{Related work}
Many authors have worked on the sample-complexity of RL in various settings. The simplest case is the multiarmed bandit
problem that has been extensively studied with varying assumptions. The typical measure of efficiency in the bandit
literature is regret, but sample-complexity bounds are also known and sometimes used. 
The next step from bandits is finite state MDPs, of which bandits are an example with only a single state.
There are two main settings when MDPs are considered, the discounted case where sample-complexity bounds are proven and the 
undiscounted (average reward) case where regret bounds are more typical. 
In the discounted setting the upper and lower bounds on sample-complexity are now extremely refined.
See \citet{Str09} for a detailed review of the popular algorithms and theorems. More recent work on closing the gap
between upper and lower bounds is by \citet{SS10,LH12,AMK12}. In the undiscounted case it is necessary to make some form of ergodicity
assumption as without this regret bounds cannot be given. In this work we avoid ergodicity assumptions and discount future 
rewards. Nevertheless, our algorithm borrows some tricks used by UCRL2 \cite{AJO10}.
Previous work for more general environment classes is somewhat limited. For factored MDPs there are known bounds, see
\cite{CS11} and references there-in.
\citet{EKM05} give essentially unimprovable exponential bounds
on the sample-complexity of learning in finite partially observable MDPs.
\citet{MNOR13} show regret bounds for undiscounted RL where the true environment is assumed to be finite, Markov and
communicating, but where the state is not directly observable.  
As far as we know there has been no work on the sample-complexity of RL when environments are completely general, but
asymptotic results have garnered some attention with
positive results by \citet{Hut02,RH08,HS12} and (mostly) negative ones by \citet{HL11b}.
Perhaps the closest related worked is \cite{DLL09}, which deals with a similar problem in the rather different setting of learning
the optimal predictor from a class of $N$ experts. They obtain an $O(N \log N)$ bound, which is applied
to the problem of structure learning for discounted finite-state factored MDPs. Our work generalises this approach to the non-Markov case
and compact model classes.

\section{Notation}

The definition of environments is borrowed from the work of \citet{Hut05}, although the notation is slightly more formal to ease the application
of martingale inequalities. 

\subsubsect{General}
$\N = \left\{0,1,2,\cdots\right\}$ is the natural numbers. For the indicator function we write $\ind{x = y} = 1$ if $x = y$ and $0$ otherwise.
We use $\wedge$ and $\vee$ for logical and/or respectively. If $A$ is a set then $|A|$ is its size and $A^*$ is the set of all finite strings (sequences) 
over $A$.
If $x$ and $y$ are sequences then $x \prefix y$ means that $x$ is a prefix of $y$.
Unless otherwise mentioned, $\log$ represents
the natural logarithm. For random variable $X$ we write $\E X$ for its expectation. For $x \in \R$, $\ceil{x}$ is the ceiling function.

\subsubsect{Environments and policies}
Let $A$, $O$ and $R \subset \R$ be finite sets of actions, observations and rewards respectively and $\H := A \times O \times R$. 
$\H^\infty$ is the set of infinite history sequences while $\H^* := (A\times O\times R)^*$ is
the set of finite history sequences. If $h \in \H^*$ then $\ell(h)$ is the number of action/observation/reward tuples in $h$.
We write $a_t(h)$, $o_t(h)$, $r_t(h)$ for the $t$th action/observation/reward of history sequence $h$.
For $h \in \H^*$, $\Gamma_{h} := \left\{h' \in \H^\infty : h \prefix h'\right\}$ is the cylinder set.
Let $\F := \sigma(\{\Gamma_{h} : h \in \H^*\})$ and $\F_t := \sigma(\{\Gamma_h : h \in \H^* \wedge \ell(h) = t\})$ be $\sigma$-algebras.
An environment $\mu$ is a set of conditional probability distributions over observation/reward pairs given the history so far.
A policy $\pi$ is a function $\pi : \H^* \to A$.
An environment and policy interact sequentially to induce a measure, $P_{\mu,\pi}$, on filtered probability space $(\H^\infty, \F, \left\{\F_t\right\})$. 
For convenience, we abuse notation and write $P_{\mu,\pi}(h) := P_{\mu,\pi}(\Gamma_{h})$.
If $h \prefix h'$ then conditional probabilities are $P_{\mu,\pi}(h'|h) := P_{\mu,\pi}(h') / P_{\mu,\pi}(h)$.
$R_t(h;d) := \sum_{k=t}^{t+d} \gamma^{k-t} r_k(h)$ is the $d$-step return function and $R_t(h) := \lim_{d \to\infty} R_t(h;d)$.
Given history $h_t$ with $\ell(h_t) = t$, the value function is defined by
$V^\pi_\mu(h_t;d) := \E[R_{t}(h;d)|h_t]$
where the expectation is taken with respect to $P_{\mu,\pi}(\cdot | h_t)$. 
$V^\pi_\mu(h_t) := \lim_{d\to\infty} V^\pi_\mu(h_t;d)$.
The optimal policy for environment $\mu$ is $\pi^*_\mu := \argmax_\pi V^\pi_\mu$, which with our assumptions is known to exist 
\cite{HL11}.
The value of the optimal policy
is $V^*_\mu := V^{\pi^*_\mu}_\mu$. 
In general, $\mu$ denotes the true environment while $\nu$ is a model. $\pi$ will typically
be the policy of the algorithm under consideration. $Q^*_\mu(h, a)$ is the value in history $h$ of following policy $\pi^*_\mu$ except
for the first time-step when action $a$ is taken. $\M$ is a set of environments (models).

\subsubsect{Sample-complexity}
Policy $\pi$ is $\epsilon$-optimal in history $h$ and environment $\mu$ if
$V^*_\mu(h) - V^\pi_\mu(h) \leq \epsilon$. 
The sample-complexity of a policy $\pi$ in environment class $\M$ is the smallest $\Lambda$ such that,
with high probability, $\pi$ is $\epsilon$-optimal for all but $\Lambda$ time-steps for all $\mu \in \M$.
Define $L_{\mu,\pi}^\epsilon:\H^\infty \to \N \cup \left\{\infty\right\}$ to be
the number of time-steps when $\pi$ is not $\epsilon$-optimal.
\eq{
L_{\mu,\pi}^\epsilon(h) := \sum_{t=1}^\infty \bigind{V^*_\mu(h_t) - V^\pi_\mu(h_t) > \epsilon},
}
where $h_t$ is the length $t$ prefix of $h$. The sample-complexity of policy $\pi$ is $\Lambda$ with respect to accuracy $\epsilon$ and 
confidence $1 - \delta$ if
$\P{L_{\mu,\pi}^\epsilon(h) > \Lambda} < \delta, \forall \mu \in \M$.

\section{Finite Case}

We start with the finite case where the true environment is known to belong to a finite set of models, $\M$.
The Maximum Exploration Reinforcement Learning algorithm is model-based in the sense that it maintains a set, $\M_t \subseteq \M$, where models
are eliminated once they become implausible. The algorithm operates in phases of exploration and exploitation, 
choosing to exploit if it knows all plausible environments are reasonably close under all optimal policies and explore otherwise. 
This method of exploration essentially guarantees that \merl{} is 
nearly optimal whenever it is exploiting and the number of exploration phases is limited with high probability.
The main difficulty is specifying what it means to be plausible. Previous authors working on finite environments, such as 
MDPs or bandits, have removed models for which the transition probabilities are not sufficiently close to their 
empirical estimates. In the more general setting this approach fails because states (histories) are never visited
more than once, so sufficient empirical estimates cannot be collected. Instead, we eliminate
environments if the reward we actually collect over time is not sufficiently close to the reward we expected given that environment.

Before giving the explicit algorithm, we explain the operation of \merl{} more formally in two parts. 
First we describe how it chooses to explore and exploit and then how the model class is maintained. See Figure \ref{fig_explore_exploit}
for a diagram of how exploration and exploitation occurs.

\subsubsect{Exploring and exploiting}
At each time-step $t$ \merl{} computes the pair of environments $\unu,\bnu$ in the model class $\M_t$ and the 
policy $\pi$ maximising the difference 
\eq{
\Delta := V^\pi_{\bnu}(h;d) - V^\pi_{\unu}(h;d), \quad d := \constd.
}
If $\Delta > \epsilon / 4$, then \merl{} follows policy $\pi$ for $d$ time-steps, which we call an exploration phase. 
Otherwise, for one time-step it follows the optimal policy
with respect to the first environment currently in the model class. Therefore, if \merl{} chooses to exploit, then all
policies and environments in the model class lead to similar values, which implies that exploiting is near-optimal. If
\merl{} explores, then either $V^\pi_{\bnu}(h;d) - V^\pi_{\mu}(h;d) > \epsilon/8$ or $V^\pi_{\mu}(h;d) - V^\pi_{\unu}(h;d) > \epsilon/8$,
which will allow us to apply concentration inequalities to eventually eliminate either $\bnu$ (the upper bound) or $\unu$ 
(the lower bound).

\subsubsect{The model class}
An exploration phase is a $\kappa$-exploration phase if $\Delta \in [\epsilon 2^{\kappa - 2}, \epsilon 2^{\kappa - 1})$, where
\eq{
\kappa \in \K := \left\{0,1,2,\cdots,\log_2 {1 \over \epsilon(1-\gamma)} + 2 \right\}.
}
For each environment $\nu \in \M$ and each $\kappa \in \K$, \merl{} associates a counter $E(\nu,\kappa)$, which is
incremented at the start of a $\kappa$-exploration phase if $\nu \in \left\{\unu,\bnu\right\}$.
At the end of each $\kappa$-exploration phase \merl{}
calculates the discounted return actually received during that exploration phase $R \in [0,1/(1 - \gamma)]$ and records the values
\eq{
X(\bnu,\kappa) &:= (1 - \gamma)(V^\pi_{\bnu}(h;d) - R) \\ 
X(\unu,\kappa) &:= (1 - \gamma)(R - V^\pi_{\unu}(h;d)), 
}
where $h$ is the history at the start of the exploration phase. So $X(\bnu,\kappa)$ is the difference between
the return expected if the true model was $\bnu$ and the actual return and $X(\unu,\kappa)$ is the difference between
the actual return and the expected return if the true model was $\unu$. Since the expected value of $R$ is $V^\pi_{\mu}(h;d)$,
and $\bnu$,$\unu$ are upper and lower bounds respectively, the expected values of both $X(\bnu,\kappa)$ and $X(\unu,\kappa)$
are non-negative and at least one of them has expectation larger than $(1 - \gamma)\epsilon / 8$. 

\merl{} eliminates environment $\nu$ from the model class if the cumulative sum of $X(\nu,\kappa)$ over all exploration
phases where $\nu \in \{\unu,\bnu\}$ is sufficiently large, but it tests this condition only when the counts $E(\nu,\kappa)$
has increased enough since the last test. 
Let $\alpha_j := \ceil{\alpha^j}$ for $\alpha \in (1, 2)$ as defined in the algorithm.
\merl{} only tests if $\nu$ should be removed from the model class when $E(\nu,\kappa) = \alpha_j$ for some $j \in \N$.
This restriction ensures that tests are not performed too often, which allows us to apply the union bound without losing too much.
Note that if the true environment $\mu \in \{\bnu,\unu\}$, then 
$\E_{\mu,\pi} X(\mu,\kappa) = 0$, which will ultimately be enough to ensure that $\mu$ remains in the model class with high
probability. The reason for using $\kappa$ to bucket exploration phases will become apparent later in the proof of 
Lemma \ref{lem_exploit_errors}.

\begin{algorithm}[H]
\caption{\merl{}}
\small
\begin{algorithmic}[1]
\STATE {\bf Inputs:} $\epsilon$, $\delta$ and $\M := \left\{\nu_1, \nu_2,\cdots,\nu_N\right\}$.
\STATE $t = 1$ and $h$ empty history
\STATE $d := \constd$, $\delta_1 := \constdeltaone$
\STATE $\alpha := \constalpha$ and $\alpha_j := \ceil{\alpha^j}$
\STATE $E(\nu,\kappa) := 0, \quad \forall \nu \in \M$ and $\kappa \in \N$
\LOOP
\REPEAT
\STATE $\Pi := \left\{\pi^*_\nu : \nu \in \M\right\}$
\STATE $\displaystyle{\bnu,\unu,\pi := \argmax_{\bnu,\unu \in \M, \pi \in \Pi} V^\pi_{\bnu}(h;d) - V^\pi_{\unu}(h;d)}$
\IF {$\Delta := V^{\pi}_{\bnu}(h;d) - V^{\pi}_{\unu}(h;d) > \epsilon / 4$}
\STATE $\tilde h = h$ and $R = 0$
\FOR{$j = 0 \to d$}
\STATE $R = R + \gamma^j r_t(h)$
\STATE {\scshape Act}($\pi$)
\ENDFOR
\STATE $\kappa := \min\left\{\kappa \in \N : \Delta > \epsilon 2^{\kappa-2}\right\}$.
\STATE $E(\unu,\kappa) = E(\unu,\kappa) + 1$ and $E(\bnu,\kappa) = E(\bnu,\kappa) + 1$
\STATE $X(\bnu,\kappa)_{E(\bnu,\kappa)} = (1 - \gamma)(V_{\bnu}^\pi(\tilde h;d) - R)$
\STATE $X(\unu,\kappa)_{E(\unu,\kappa)} = (1 - \gamma)(R - V_{\unu}^\pi(\tilde h;d))$
\ELSE
\STATE $i := \min \left\{i : \nu_i \in \M\right\}$ and {\scshape Act}($\pi^*_{\nu_i}$)
\ENDIF
\UNTIL{$\exists \nu \in \M, \kappa,j \in \N$ such that $E(\nu,\kappa) = \alpha_j$ and 
\eq{
\sum_{i=1}^{E(\nu,\kappa)} X(\nu,\kappa)_i \geq \sqrt{2E(\nu,\kappa) \log{E(\nu,\kappa) \over \delta_1}}
.}}
\STATE $\M = \M - \left\{\nu\right\}$
\ENDLOOP
\STATE {\bf function} {\scshape Act}($\pi$)
\STATE \hspace{0.25cm} Take action $a_t = \pi(h)$ and receive reward and observation $r_t, o_t$ from environment
\STATE \hspace{0.25cm} $t \leftarrow t + 1$ and $h \leftarrow ha_to_tr_t$
\STATE {\bf end function}
\end{algorithmic}
\end{algorithm}

\subsubsect{Subscripts}
For clarity, we have omitted subscripts in the pseudo-code above. In the analysis we will refer to $E_t(\nu,\kappa)$ and $\M_t$ for the
values of $E(\nu,\kappa)$ and $\M$ respectively at time-step $t$. We write $\nu_t$ for $\nu_i$ in line 21 and similarly $\pi_t := \pi^*_{\nu_t}$.

\subsubsect{Phases}
An {\it exploration phase} is a period of exactly $d$ time-steps, starting at time-step $t$ if
\begin{enumerate}
\item $t$ is not currently in an exploration phase.
\item $\Delta := V^{\pi}_{\bnu}(h_t;d) - V^{\pi}_{\unu}(h_t;d) > \epsilon / 4$.
\end{enumerate}
We say it is a $\nu$-exploration phase if 
$\nu = \unu$ or $\nu = \bnu$ and a $\kappa$-exploration phase if $\Delta \in [\epsilon 2^{\kappa-2}, \epsilon 2^{\kappa-1})
\equiv [\epsilon_\kappa, 2\epsilon_\kappa)$ where $\epsilon_\kappa := \epsilon2^{\kappa - 2}$.
It is a $(\nu,\kappa)$-exploration phase if it satisfies
both of the previous statements. We say that \merl{} is {\it exploiting} at time-step $t$ if $t$ is not in an exploration phase. 
A {\it failure} phase is also a period of $d$ time-steps and starts in time-step $t$ if
\begin{enumerate}
\item $t$ is not in an exploration phase or earlier failure phase
\item $V^*_\mu(h_t) - V^\pi_\mu(h_t) > \epsilon$.
\end{enumerate}
Unlike exploration phases, the algorithm does not depend on the failure phases, which are only used in the analysis,
An exploration or failure phase starting at time-step $t$ is {\it proper} if $\mu \in \M_t$.
The effective horizon $d$ is chosen to ensure that $V^\pi_\mu(h;d) \geq V^\pi_\mu(h) - \epsilon / 8$ for all $\pi$, $\mu$ and $h$.

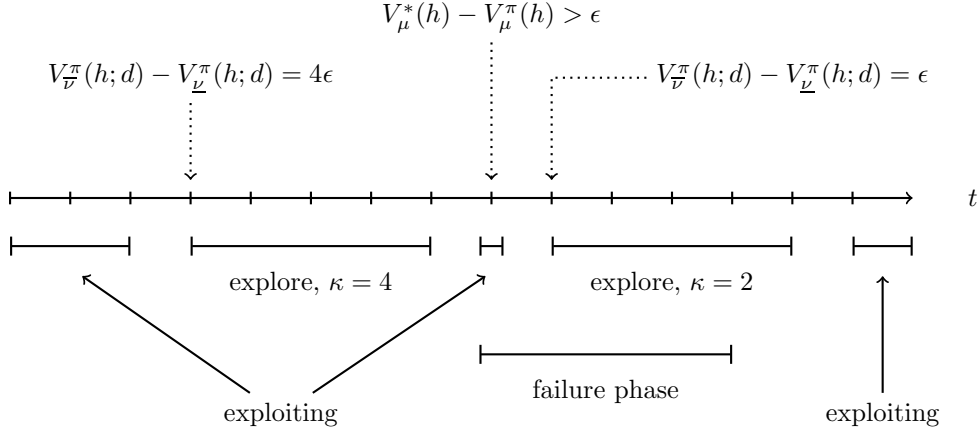
\begin{figure}[H]
\centering
\begin{tikzpicture}[thick,scale=0.8]
\tikzstyle{arrow} = [->,=>stealth']
\draw (0,0) edge[->] (15,0);
\foreach \n in {0,...,14} {
	\draw (\n,-3pt) -- (\n,3pt);
}
\node at (16,0) {$t$};

\node (explore1) at (3,2) {$V^\pi_{\bnu}(h;d) - V^\pi_{\unu}(h;d) = 4\epsilon$};
\node (explore2) at (13,2) {$V^\pi_{\bnu}(h;d) - V^\pi_{\unu}(h;d) = \epsilon$};

\node (failure) at (8,3) {$V^*_\mu(h) - V^\pi_\mu(h) > \epsilon$};

\draw[arrow,dotted] (explore2) -| (9,0.3);
\draw[arrow,dotted] (explore1) -- (3,0.3);
\draw[arrow,dotted] (failure) -- (8,0.3);

\draw(0,-0.8) edge[|-|] (2,-0.8);
\draw(3,-0.8) edge[|-|] node[below=6pt] {explore, $\kappa = 4$} (7,-0.8);
\draw(7.8,-0.8) edge[|-|] (8.2,-0.8);
\draw(9,-0.8) edge[|-|] node[below=6pt] {explore, $\kappa = 2$} (13,-0.8);
\draw(14,-0.8) edge[|-|] (15,-0.8);

\draw(7.8,-2.6) edge[|-|] node[below=6pt] {failure phase} (12,-2.6);

\node (exploit) at(4.5,-3.6) {exploiting};
\node (exploit2) at(14.5,-3.6) {exploiting};

\draw (exploit) edge[->] (1.2,-1.3); 
\draw (exploit) edge[->] (7.9,-1.3); 
\draw (exploit2) edge[->] (14.5,-1.3); 

\end{tikzpicture}

\caption{Exploration/exploitation/failure phases, $d = 4$}
\label{fig_explore_exploit}

\end{figure}

\subsubsect{Test statistics}
We have previously remarked that most traditional model-based algorithms with sample-complexity guarantees record statistics about the transition probabilities of an environment. Since the environments
are assumed to be finite, these statistics eventually become accurate (or irrelevant) and the standard theory on the concentration of measure can be used for hypothesis testing.
In the general case, environments can be infinite and so we cannot collect useful statistics about individual transitions. Instead, we use the statistics $X(\nu,\kappa)$, which are dependent on the value
function rather than individual transitions.
These satisfy $\E_{\mu,\pi}[X(\mu,\kappa)_i] = 0$ while $\E_{\mu,\pi}[X(\nu,\kappa)_i] \geq 0$ for all $\nu \in \M_t$. Testing is then performed on 
the statistic $\sum_{i=1}^{\alpha_k} X(\nu,\kappa)_i$, which will satisfy certain martingale inequalities.

\subsubsect{Updates}
As \merl{} explores, it updates its model class, $\M_t \subseteq \M$, by removing environments that have become implausible.
This is comparable to the updating of confidence intervals for algorithms such as MBIE (\citealp{LS05}) or UCRL2 (\citealp{AJO10}).
In MBIE, the confidence interval about the empirical estimate of a transition probability is updated after every observation. 
A slight theoretical improvement used by UCRL2 is to only update when the number of samples of a particular statistic doubles.
The latter trick allows a cheap application of the union bound over all updates without wasting too many samples.
For our purposes, however, we need to update slightly more often than the doubling trick would allow. Instead,
we check if an environment should be eliminated if the number of $(\nu,\kappa)$-exploration phases is exactly $\alpha_j$ for some $j$ where
$\alpha_j := \ceil{\alpha^j}$ and $\alpha := \constalpha \in (1, 2)$. Since the growth of $\alpha_j$ is still exponential, the union
bound will still be applicable.

\subsubsect{Probabilities} For the remainder of this section, unless otherwise mentioned, all probabilities and expectations are with respect to $P_{\mu,\pi}$ where $\pi$ is the policy of 
Algorithm 1 and $\mu \in \M$ is the true environment.

\subsubsect{Analysis}
Define $\gmax := \constgmax$ and $\emax := \constemax$, which are high probability bounds on the number of 
failure and exploration phases respectively. 
\begin{theorem}\label{thm_main}
Let $\mu \in \M = \left\{\nu_1,\nu_2,\cdots \nu_N\right\}$ be the true environment and $\pi$ be the policy of Algorithm 1. Then
\eq{
\P{L_{\mu,\pi}^\epsilon(h) \geq  d\cdot (\gmax + \emax)} \leq \delta.
}
\end{theorem}

If lower order logarithmic factors are dropped then the sample-complexity bound of \merl{} given by Theorem \ref{thm_main} is $\aO{{N \over \epsilon^2(1 - \gamma)^3} \log^2{N \over \delta\epsilon(1 - \gamma)}}$.
Theorem \ref{thm_main} follows from three lemmas.

\begin{lemma}\label{lem_model}
$\mu \in \M_t$ for all $t$ with probability $1 - \delta / 4$.
\end{lemma}

\begin{lemma}\label{lem_exploit_errors}
The number of proper failure phases is bounded by
\eq{
\gmax := \constgmax
}
with probability at least $1 - {\delta \over 2}$.
\end{lemma}

\begin{lemma}\label{lem_explore}
The number of proper exploration phases is bounded by 
\eq{
\emax := \constemax
}
with probability at least $1 - {\delta \over 4}$.
\end{lemma}

\begin{proofof}{ of Theorem \ref{thm_main}}
Applying the union bound to the results of Lemmas \ref{lem_model}, \ref{lem_exploit_errors} and \ref{lem_explore} gives the following
with probability at least $1 - \delta$.
\begin{enumerate}
\item There are no non-proper exploration or failure phases.
\item The number of proper exploration phases is at most $\emax$.
\item The number of proper failure phases is at most $\gmax$.
\end{enumerate}
If $\pi$ is not $\epsilon$-optimal at time-step $t$ then $t$ is either in an exploration or failure phase.
Since both are exactly $d$ time-steps long the total number of time-steps when $\pi$ is sub-optimal is at most
$d \cdot (\gmax + \emax)$. 
\end{proofof}

We now turn our attention to proving Lemmas \ref{lem_model}, \ref{lem_exploit_errors} and \ref{lem_explore}. Of these,
Lemma \ref{lem_explore} is more conceptually challenging while Lemma \ref{lem_exploit_errors} is 
intuitively unsurprising, but technically difficult.

\begin{proofof}{ of Lemma \ref{lem_model}} 
If $\mu$ is removed from $\M$, then there exists a $\kappa$ and $j \in \N$ such that
\eq{
\sum_{i=1}^{\alpha_j} X(\mu,\kappa)_i \geq \sqrt{2\alpha_j \log {\alpha_j \over \delta_1}}.
}
Fix a $\kappa \in \K$, $E_\infty(\mu,\kappa) := \lim_t E_t(\mu,\kappa)$ and $X_i := X(\mu,\kappa)_i$. 
Define a sequence of random variables
\eq{
\tilde X_i := \begin{cases}
X_i & \text{if } i \leq E_\infty(\mu,\kappa) \\
0 & \text{otherwise}.
\end{cases}
}
Now we claim that $B_n := \sum_{i=1}^{n} \tilde X_i$ is a martingale with $|B_{i+1} - B_i| \leq 1$ and $\E B_i = 0$.
That it is a martingale with zero expectation follows because if $t$ is the time-step at the start of the exploration phase associated with variable $X_i$, then
$\E[X_i|\F_t] = 0$. $|B_{i+1} - B_i| \leq 1$ because discounted returns are bounded in $[0, 1/(1 - \gamma)]$ and by the definition of $X_i$.

For all $j \in \N$, we have by Azuma's inequality that
\eq{
\P{B_{\alpha_j} \geq \sqrt{2\alpha_j \log{\alpha_j\over \delta_1}}} \leq {\delta_1 \over \alpha_j}.
}
Apply the union bound over all $j$. 
\eq{
\P{\exists j \in \N : B_{\alpha_j} \geq \sqrt{2\alpha_j \log{\alpha_j\over \delta_1}}} \leq \sum_{j=1}^\infty {\delta_1 \over \alpha_j}.
}
Complete the result by the union bound over all $\kappa$,
applying Lemma \ref{lem_tech2} (see Appendix) and the definition of $\delta_1$ to bound $\sum_{\kappa \in \K} \sum_{j=1}^\infty {\delta_1 \over \alpha_j} \leq \delta / 4$.
\end{proofof}

We are now ready to give a high-probability bound on the number of proper exploration phases. 
If \merl{} starts a proper exploration phase at time-step $t$ then at least one of the following holds:
\begin{enumerate}
\item $\E[X(\unu,\kappa)_{E(\unu,\kappa)}|\F_t] > (1 - \gamma) \epsilon / 8$.
\item $\E[X(\bnu,\kappa)_{E(\bnu,\kappa)}|\F_t] > (1 - \gamma) \epsilon / 8$.
\end{enumerate}
This contrasts with $\E[X(\mu,\kappa)_{E(\mu,\kappa)}|\F_t] = 0$, which ensures that $\mu$ remains in $\M$ for all time-steps.
If one could know which of the above statements were true at each time-step then it would be comparatively easy to show by means of Azuma's inequality that
all environments that are not $\epsilon$-close are quickly eliminated after $O({1 \over \epsilon^2(1 - \gamma)^2})$ $\nu$-exploration phases, which
would lead to the desired bound. Unfortunately though, the truth of (1) or (2) above cannot be determined,
which greatly increases the complexity of the proof.

\begin{proofof}{ of Lemma \ref{lem_explore}}
Fix a $\kappa \in \K$ and let $\emaxkappa$ be a constant to be chosen later.
Let $h_t$ be the history at the start of some $\kappa$-exploration phase. We say an $(\unu,\kappa)$-exploration phase is $\unu$-effective if
\eq{
\E[X(\unu,\kappa)_{E(\unu,\kappa)}|\F_t]
&\equiv (1 - \gamma)(V_{\mu}^\pi(h_t;d) - V_{\unu}^\pi(h_t;d)) \\
&> (1 - \gamma)\epsilon_\kappa / 2 
}
and $\bnu$-effective if the same condition holds for $\bnu$.
Now since $t$ is the start of a proper exploration phase we have that $\mu \in \M_t$ and so
\eq{
&V^\pi_{\bnu}(h_t;d) \geq V^\pi_{\mu}(h_t;d) \geq V^\pi_{\unu}(h_t;d) \\ 
&V^\pi_{\bnu}(h_t;d) - V^\pi_{\unu}(h_t;d) > \epsilon_\kappa. 
}
Therefore every proper exploration phase is either $\unu$-effective or $\bnu$-effective.
Let $E_{t,\kappa} := \sum_{\nu} E_t(\nu,\kappa)$, which is twice the number of $\kappa$-exploration phases at time $t$ and 
$E_{\infty,\kappa} := \lim_t E_{t,\kappa}$, which is twice the total
number of $\kappa$-exploration phases.\footnote{Note that it is never the case that $\bnu = \unu$ at the start of an exploration phase, since in this case $\Delta = 0$.} 
Let $F_t(\nu,\kappa)$ be the number of $\nu$-effective $(\nu,\kappa)$-exploration phases up to time-step $t$.
Since each proper $\kappa$-exploration phase is either $\unu$-effective or $\bnu$-effective or both, 
$\sum_{\nu} F_t(\nu,\kappa) \geq E_{t,\kappa} / 2$.
Applying Lemma \ref{lem_tech1} to $y_\nu := E_t(\nu,\kappa) / E_{t,\kappa}$ and $x_\nu := F_t(\nu,\kappa) / E_t(\nu,\kappa)$ shows
that if $E_{\infty,\kappa} > \emaxkappa$ then there exists a $t'$ and $\nu$ such that $E_{t',\kappa} = \emaxkappa$ and
\eqn{
\label{hard_cond1}
{F_{t'}(\nu,\kappa)^2 \over \emaxkappa E_{t'}(\nu,\kappa)} \geq {1 \over 4N},
}
which implies that 
\eqn{
\label{hard_cond2}
F_{t'}(\nu,\kappa) \geq \sqrt{\emaxkappa E_{t'}(\nu,\kappa) \over 4N} \stackrel{(a)}\geq {E_{t'}(\nu,\kappa) \over \sqrt{4N}},
}
where (a) follows because $\emaxkappa = E_{t',\kappa} \geq E_{t'}(\nu,\kappa)$.
Let $Z(\nu)$ be the event that there exists a $t'$ satisfying (\ref{hard_cond1}).
We will shortly show that $\P{Z(\nu)} < \delta / (4N|\K|)$. Therefore
\eq{
\P{E_{\infty,\kappa} > \emaxkappa} 
&\leq \P{\exists \nu : Z(\nu)} 
\leq \sum_{\nu \in \M} \P{Z(\nu)} \\
&\leq \delta / (4|\K|)
}
Finally take the union bound over all $\kappa$ and let 
\eq{
\emax := \sum_{\kappa \in \K} {1\over 2} \emaxkappa,
}
where we used ${1 \over 2}\emaxkappa$ because $\emaxkappa$ is a high-probability upper bound on $E_{\infty,\kappa}$, which is {\it twice} the number of $\kappa$-exploration
phases.

\subsubsect{Bounding $\P{Z(\nu)} < \delta / (4N|\K|)$}
Fix a $\nu \in \M$ and let $X_1, X_2, \cdots, X_{E_\infty(\nu,\kappa)}$ be the sequence with $X_i := X(\nu,\kappa)_i$ and let $t_i$ 
be the corresponding time-step at the start of the $i$th $(\nu,\kappa)$-exploration phase. Define a sequence
\eq{
Y_i := \begin{cases}
X_i - \E[X_i|\F_{t_i}] & \text{if } i \leq E_\infty(\nu,\kappa) \\
0 & \text{otherwise}
\end{cases}
}
Let $\lambda(E) := \sqrt{2E \log {E \over \delta_1}}$. 
Now if $Z(\nu)$, then the largest time-step $t \leq t'$ with $E_t(\nu,t) = \alpha_j$ for some $j \in \N$ is
\eq{
t := \max \left\{t \leq t' :  \exists j \in \N \text{ s.t. } \alpha_j = E_t(\nu,t) \right\},
}
which exists and satisfies 
\begin{enumerate}
\item $E_t(\nu,\kappa) = \alpha_j$ for some $j$.
\item $E_\infty(\nu,\kappa) > E_t(\nu,\kappa)$.
\item $F_t(\nu,\kappa) \geq \sqrt{E_t(\nu,\kappa) \emaxkappa / (16N)}$.
\item $E_t(\nu,\kappa) \geq {\emaxkappa / (16N)}$.
\end{enumerate}
where parts 1 and 2 are straightforward and parts 3 and 4 follow by the definition of $\left\{\alpha_j\right\}$, which was chosen
specifically for this part of the proof.
Since $E_\infty(\nu,\kappa) > E_t(\nu,\kappa)$, at the end of the exploration phase starting at time-step $t$, $\nu$ must remain in $\M$. Therefore
\eqn{
\nonumber \lambda(\alpha_j)
&\stackrel{(a)}\geq \sum_{i=1}^{\alpha_j} X_i 
\stackrel{(b)}\geq  \sum_{i=1}^{\alpha_j} Y_i + {\epsilon_\kappa(1 - \gamma) F_t(\nu,\kappa) \over 2} \\
\label{eq-10} &\stackrel{(c)}\geq \sum_{i=1}^{\alpha_j} Y_i + {\epsilon_\kappa(1 - \gamma)\over 8}\sqrt{\alpha_j \emaxkappa \over N}, 
}
where in (a) we used the definition of the confidence interval of \merl{}.
In (b) we used the definition of $Y_i$ and the fact that $\E X_i \geq 0$ for all $i$ and 
$\E X_i \geq \epsilon_\kappa(1 - \gamma) / 2$ if $X_i$ is effective. Finally we used the lower bound on the number of effective $\nu$-exploration
phases, $F_t(\nu,\kappa)$ (part 3 above).
If $\emaxkappa := \constemaxkappa$, then by applying Lemma \ref{lem_w} with $a = {2^{9} N \over \epsilon^2_\kappa(1 -\gamma)^2}$ and $b = 1/\delta_1$ we obtain
\eq{
\emaxkappa \geq {2^9 N \over \epsilon_\kappa^2 (1 - \gamma)^2} \log {\emaxkappa \over \delta_1}
\geq {2^9 N \over \epsilon_\kappa^2 (1 - \gamma)^2} \log {\alpha_j \over \delta_1}
}
Multiplying both sides by $\alpha_j$ and rearranging and using the definition of $\lambda(\alpha_j)$ leads to
\eq{
{\epsilon_\kappa(1 - \gamma) \over 8}\sqrt{\alpha_j \emaxkappa \over N} \geq 2\lambda(\alpha_j).
}
Inserting this into \eqr{eq-10} shows that 
$Z(\nu)$ implies that there exists an $\alpha_j$ such that
$\sum_{i=1}^{\alpha_j} Y_i \leq - \lambda(\alpha_j)$. Now by the same argument as in the proof of Lemma \ref{lem_model}, 
$B_n := \sum_{i=1}^{n} Y_i$ is a martingale with $|B_{i+1} - B_i| \leq 1$.
Therefore by Azuma's inequality
\eq{
\P{\sum_{i=1}^{\alpha_j} Y_i \leq -\lambda(\alpha_j)} 
&\leq {\delta_1 \over \alpha_j}. 
}
Finally apply the union bound over all $j$.
\end{proofof}

Recall that if \merl{} is exploiting at time-step $t$, then $\pi_t$ is the optimal policy with respect to the first environment
in the model class. 
To prove Lemma \ref{lem_exploit_errors} we start by showing that in this case $\pi_t$
is nearly-optimal.

\begin{lemma}\label{lem_exploit}
Let $t$ be a time-step and $h_t$ be the corresponding history. If $\mu \in \M_t$ and \merl{} is exploiting (not exploring), then
$V^*_\mu(h_t) - V^{\pi_t}_\mu(h_t) \leq 5\epsilon/8$.
\end{lemma}

\begin{proofof}{ of Lemma \ref{lem_exploit}}
Since \merl{} is not exploring 
\eq{
V^{*}_\mu(h_t) - V^{\pi_t}_\mu(h_t) 
&\stackrel{(a)}\leq V^{*}_\mu(h_t;d) - V^{\pi_t}_\mu(h_t;d) + {\epsilon \over 8} \\
&\stackrel{(b)}\leq V^{\pi^*_{\mu}}_{\nu_t}(h_t;d) - V^{\pi_t}_{\nu_t}(h_t;d) + 5\epsilon/8 \\
&\stackrel{(c)}\leq 5\epsilon / 8,
}
(a) follows by truncating the value function. (b) follows because $\mu \in \M_t$ and \merl{} is exploiting.
(c) is true since $\pi_t$ is the optimal policy in $\nu_t$.
\end{proofof}

Lemma \ref{lem_exploit} is almost sufficient to prove Lemma \ref{lem_exploit_errors}. The only problem is that \merl{} 
only follows $\pi_t = \pi^*_{\nu_t}$ until there is an exploration phase. The idea to prove Lemma
\ref{lem_exploit_errors} is as follows:
\begin{enumerate}
\item If there is a low probability of entering an exploration phase
within the next $d$ time-steps following policy $\pi_t$, then $\pi$ is nearly as good as $\pi_t$, which itself
is nearly optimal by Lemma \ref{lem_exploit}. 
\item The number of time-steps when the probability of entering an exploration phase within the
next $d$ time-steps is high is unlikely to be too large before an exploration phase is triggered. Since there are not many exploration
phases with high probability, there are also unlikely to be too many time-steps when $\pi$ expects to enter one with high probability.
\end{enumerate}
Before the proof of Lemma \ref{lem_exploit_errors} we remark on an easier
to prove (but weaker) version of Theorem \ref{thm_main}. If \merl{} is exploiting then Lemma \ref{lem_exploit} shows that $V^*_\mu(h) - Q^*_\mu(h, \pi(h)) \leq 5\epsilon/8 < \epsilon$.
Therefore if we cared about the number of time-steps when this is not the case (rather than $V^*_\mu - V^\pi_\mu$), then we would already be done by combining 
Lemmas \ref{lem_explore} and \ref{lem_exploit}.

\begin{proofof}{ of Lemma \ref{lem_exploit_errors}}
Let $t$ be the start of a proper failure phase with corresponding history, $h$.
Therefore $V^*_\mu(h) - V^\pi_\mu(h) > \epsilon$. By Lemma \ref{lem_exploit}, 
$V^*_\mu(h) - V^\pi_\mu(h) = V^*_\mu(h) - V^{\pi_t}_\mu(h) + V^{\pi_t}_\mu(h) - V^{\pi}_\mu(h) \leq 5\epsilon / 8 + V^{\pi_t}_\mu - V^\pi_\mu(h)$ and so 
\eqn{
\label{eq-11} V^{\pi_t}_\mu(h) - V^\pi_\mu(h) \geq {3\epsilon \over 8}.
}
We define set $\H_{\kappa} \subset \H^*$ to be the set of extensions of $h$ that trigger $\kappa$-exploration phases. Formally
$\H_{\kappa} \subset \H^*$ is the prefix free set such that $h'$ in $\H_\kappa$ if $h \prefix h'$ and $h'$ triggers a 
$\kappa$-exploration phase for the first time since $t$.
Let 
$\H_{\kappa,d} := \left\{h' : h' \in \H_\kappa \wedge \ell(h') \leq t + d\right\}$, which is the set of extensions of $h$ that
are at most $d$ long and trigger $\kappa$-exploration phases.
Therefore 
\eq{
{3\epsilon \over 8} 
&\stackrel{(a)}\leq V^{\pi_t}_\mu(h) - V^\pi_\mu(h) \\ 
&\stackrel{(b)}= \sum_{\kappa \in \K} \sum_{h' \in \H_\kappa} P(h'|h) \gamma^{\ell(h') - t} \left(V^{\pi_t}_\mu(h') - V^\pi_\mu(h')\right) \\
&\stackrel{(c)}\leq \sum_{\kappa \in \K} \sum_{h' \in \H_{\kappa,d}} P(h'|h) \left(V^{\pi_t}_\mu(h') - V^\pi_\mu(h')\right) + {\epsilon \over 8}\\
&\stackrel{(d)}\leq \sum_{\kappa \in \K} \sum_{h' \in \H_{\kappa,d}} P(h'|h) \left(V^{*}_\mu(h';d) - V^\pi_\mu(h';d)\right) + {\epsilon \over 4} \\
&\stackrel{(e)}\leq \sum_{\kappa \in \K} \sum_{h' \in \H_{\kappa,d}} P(h'|h) 4\epsilon_\kappa + {\epsilon \over 4}, 
}
(a) follows from \eqr{eq-11}. (b) by noting that that $\pi = \pi_t$ until an exploration phase is triggered. (c)
by replacing $\H_{\kappa}$ with $\H_{\kappa,d}$ and noting that if $h' \in \H_{\kappa} - \H_{\kappa,d}$, then 
$\gamma^{\ell(h') - t} \leq (1 - \gamma) \epsilon / 8$.
(d) by substituting $V^*_\mu(h') \geq V^{\pi_t}_\mu(h')$ and by using the effective horizon to truncate the value functions.
(e) by the definition of a $\kappa$-exploration phase.

Since the maximum of a set is greater than the average, there exists a $\kappa \in \K$ such that $\sum_{h' \in \H_{\kappa,d}} P(h'|h) \geq 2^{-\kappa-3} / |\K|$, which is the probability
that MERL encounters a $\kappa$-exploration phase within $d$ time-steps from $h$.
Now fix a $\kappa$ and let $t_1,t_2,\cdots,\cdots,t_{G_\kappa}$ be the sequence of time-steps such that
$t_i$ is the start of a failure phase and 
the probability of a $\kappa$-exploration phase within the next $d$ time-steps is at least $2^{-\kappa-3} / |\K|$.
Let $Y_i \in \left\{0, 1\right\}$ be the event that a $\kappa$-exploration phase does occur within $d$ time-steps of $t_i$ and define an auxiliary infinite sequence 
 $\tilde Y_1, \tilde Y_2, \cdots$ by
$\tilde Y_i := 
Y_i$ if $i \leq G_\kappa$ and $1$ otherwise.
Let $E_\kappa$ be the number of $\kappa$-exploration phases and
$\gmaxkappa$ be a constant to be chosen later and suppose $G_\kappa > \gmaxkappa$, then $\sum_{i=1}^{\gmaxkappa} \tilde Y_i = \sum_{i=1}^{\gmaxkappa} Y_i$ 
and either $\sum_{i=1}^{\gmaxkappa} \tilde Y_i \leq \emaxkappa$ or $E_\kappa > \emaxkappa$, where the latter follows because $Y_i = 1$ implies a $\kappa$-exploration
phase occurred.
Therefore 
\eq{
&\P{G_\kappa > \gmaxkappa} \\ 
&\leq \P{\sum_{i=1}^{\gmaxkappa} \tilde Y_i < \emaxkappa} + \P{E_\kappa > \emaxkappa} \\ 
&\leq \P{\sum_{i=1}^{\gmaxkappa} \tilde Y_i < \emaxkappa} + {\delta \over 4|\K|}. 
}
We now choose $\gmaxkappa$ sufficiently large to bound the first term in the display above by $\delta / (4|\K|)$. By the definition
of $\tilde Y_i$ and $Y_i$, if $i \leq G_\kappa$ then $\E[\tilde Y_i|\F_{t_i}] \geq 2^{-\kappa-3} / |\K|$ and for $i > G_\kappa$, $\tilde Y_i$ is always $1$.
Setting 
\eq{
\gmaxkappa 
&:= 2^{\kappa + 4} |\K| \emaxkappa \\
&= \constgmaxkappa
}
is sufficient to guarantee $\E[\sum_{i=1}^{\gmaxkappa} \tilde Y_i] > 2\emaxkappa$ and an application of Azuma's inequality
to the martingale difference sequence completes the result.
Finally we apply the union bound over all $\kappa$ and set $\gmax := \sum_{\kappa \in \N} \gmaxkappa > \sum_{\kappa \in \K} \gmaxkappa$.
\end{proofof}

\section{Compact Case}
In the last section we presented \merl{} and proved a sample-complexity bound for the case when the environment class is 
finite. In this section we show that if the number of environments is infinite, but compact with respect to the topology generated
by a natural metric, then
sample-complexity bounds are still possible with a minor modification of \merl{}. The key idea is to use compactness to
cover the space of environments with $\epsilon$-balls and compute statistics on these balls rather than individual environments.
Since all environments in the same $\epsilon$-ball are sufficiently close, the resulting statistics cannot be significantly
different and all analysis goes through identically to the finite case.
Define a topology on the space of all environments induced by the pseudo-metric 
\eq{
d(\nu_1, \nu_2) := \sup_{h, \pi} |V^\pi_{\nu_1}(h) - V^\pi_{\nu_2}(h)|.
}
\begin{theorem}
Let $\M$ be compact and coverable by $N$ $\epsilon$-balls then a modification of
Algorithm 1 satisfies
\eq{
\P{L_{\mu,\pi}^{2\epsilon}(h) \geq  d\cdot (\gmax + \emax)} \leq \delta.
}
\end{theorem}

The main modification is to define statistics on elements of the cover, rather than specific environments.
\begin{enumerate}
\item Let $U_1, \cdots, U_N$ be an $\epsilon$-cover of $\M$.
\item At each time-step choose $\uU$ and $\oU$ such that $\unu \in \underline U$ and $\bnu \in \overline U$.
\item Define statistics $\left\{X \right\}$ on elements of the cover, rather than environments, by 
\eq{
X(\uU, \kappa)_{E(\uU,\kappa)} &:= \inf_{\unu \in \uU} (1 - \gamma)(R - V^\pi_{\unu}(h)) \\
X(\oU, \kappa)_{E(\oU,\kappa)} &:= \inf_{\bnu \in \oU} (1 - \gamma)(V^\pi_{\bnu}(h) - R)
}
\item If there exists a $U$ where the test fails then eliminate all environments in that cover.
\end{enumerate}
The proof requires only small modifications to show that with high probability the $U$ containing the true environment is never
discarded, while those not containing the true environment are if tested sufficiently often.

\section{Unbounded Environment Classes}

If the environment class is non-compact then we cannot in general expect finite sample-complexity bounds. Indeed, even
asymptotic results are usually not possible.

\begin{theorem}
There exist non-compact $\M$ for which no agent has a finite \pac{} bound.
\end{theorem}

The obvious example is when $\M$ is the set of all environments. Then for any
policy $\M$ includes an environment that is tuned to ensure the policy acts sub-optimally infinitely often.
A more interesting example is the class of all computable environments, which is non-compact and also does not admit
algorithms with uniform finite sample-complexity. See negative results by \citet{HL11b} for counter-examples.

\section{Lower Bound}

We now turn our attention to the lower bound. In specific cases, the bound in Theorem \ref{thm_main} is
very weak. For example, if $\M$ is the class of finite MDPs with $|S|$ states then a natural covering leads to a \pac{} 
bound with exponential dependence on the state-space
while it is known that the true dependence is at most quadratic.
This should not be surprising since information about the transitions for
one state gives information about a large subset of $\M$, not just a single environment.
We show that the bound in Theorem \ref{thm_main} is unimprovable for general
environment classes except for logarithmic factors. That is, there exists a class of 
environments where Theorem \ref{thm_main} is nearly tight.

The simplest counter-example is a set of MDPs with four states, $S = \left\{0,1, \oplus, \ominus \right\}$ and $N$ 
actions, $A = \left\{a_1,\cdots, a_N\right\}$. 
The rewards and transitions are depicted in Figure \ref{fig:counter} where the transition probabilities depend on the action.
Let $\M := \left\{\nu_1, \cdots, \nu_N\right\}$ where for $\nu_k$ we set 
$\epsilon(a_i) = \ind{i = k}\epsilon(1 - \gamma)$. Therefore in environment $\nu_k$, $a_k$ is the optimal action in state $1$.
$\M$ can be viewed as a set of bandits with rewards in $(0, 1/(1 - \gamma))$. In the bandit domain tight lower bounds
on sample-complexity are known and given in \citet{MT04}. These results can be applied as in \citet{Str09} and \citet{LH12} to show that
no algorithm has sample-complexity
less than $O({N \over \epsilon^2(1 - \gamma)^3} \log{1 \over \delta})$.
\begin{figure}[H]
\centering
\begin{tikzpicture}[->,>=stealth',shorten >=1pt,auto,node distance=3cm, semithick,scale=0.7]
\tikzstyle{state} = [circle,draw,minimum width=0.8cm, minimum height=0.8cm,inner sep=1pt]
\node[state] (error) {$\stackrel{1} {r = 0}$};
\node[state] (minus) [right of=error] {$\stackrel{\ominus} {r = 0}$};
\node[state] (plus) [below of=error] {$\stackrel{\oplus} {r = 1}$};
\node[state] (hard) [right of=plus] {$\stackrel{0} {r = 0}$};
\path (hard) edge node[swap] {$1 - p$} (error)
      (hard) edge[loop below] node {$p := 1/(2 - \gamma)$} (hard)
      (error) edge node {${1 \over 2} - \epsilon(a)$} (minus)
      (error) edge node[right=-2pt] {${1 \over 2} + \epsilon(a)$} (plus)
      (minus) edge[loop above] node[swap] {$q := 2 - 1/\gamma$} (minus)
      (minus) edge node {$1-q$} (hard)
      (plus) edge[loop below] node[swap] {$q$} (plus)
      (plus) edge node[swap] {$1 - q$} (hard)
;
\end{tikzpicture}
\caption{Counter-example}\label{fig:counter}
\end{figure}
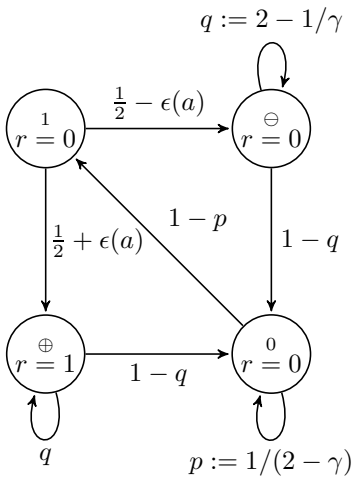

\section{Conclusions}

\subsubsect{Summary}
The Maximum Exploration Reinforcement Learning algorithm was presented. For finite classes of arbitrary environments a sample-complexity
bound was given that is linear in the number of environments. We also presented lower bounds that show that in general this
cannot be improved except for logarithmic factors. Learning is also possible for compact classes with
the sample complexity depending on the size of the smallest $\epsilon$-cover where the distance between two environments is the
difference in value functions over all policies and history sequences.
Finally, for non-compact classes of environments sample-complexity bounds are typically not possible.

\subsubsect{Running time}
The running time of \merl{} can be arbitrary large since computing the policy maximising $\Delta$ depends on the environment class used.
Even assuming the distribution of observation/rewards given the history can be computed in constant time, the values of optimal
policies can still only be computed in time exponential in the horizon. 

\subsubsect{Future work}
\merl{} is close to unimprovable in the sense that there exists a class of environments where the upper bound is nearly
tight. On the other hand, there are classes of environments where the bound of Theorem \ref{thm_main} scales badly compared to the
bounds of tuned algorithms (for example, finite state MDPs).
It would be interesting to show that \merl{}, or a variant thereof, actually performs comparably to the optimal
sample-complexity even in these cases. This question is likely to be subtle since there are unrealistic classes of environments where the
algorithm minimising sample-complexity should take actions leading directly to a trap where it receives low reward eternally, but is never (again)
sub-optimal. Since \merl{} will not behave this way it will tend to have poor sample-complexity bounds in this type of environment class.
This is really a failure of the sample-complexity optimality criterion rather than \merl{}, since jumping into non-rewarding traps is
clearly sub-optimal by any realistic measure.

\subsubsect{Acknowledgements} This work was supported by ARC grant DP120100950.


{\addcontentsline{toc}{section}{\refname}
\bibliographystyle{plainnat}
 \small\bibliography{agi-pac}
}

\appendix
\section{Technical Results}

\begin{lemma}\label{lem_tech1}
Let $x,y \in [0,1]^N$ satisfy $\sum_{i=1}^N y_i = 1$ and $\sum_{i=1}^N x_i y_i \geq 1/2$. Then 
$\max_i x_i^2 y_i > 1/(4N)$.
\end{lemma}

\begin{proof}
The result essentially follows from the fact that a maximum is greater than an average.
\eq{
\sum_{i=1}^N x_i^2 y_i &= \sum_{i=1}^N x_i y_i - \sum_{i=1}^N x_i y_i(1 - x_i) \\
&\geq{1 \over 2} - \sum_{i=1}^N x_iy_i(1 - x_i) 
\geq{1 \over 2} - \sum_{i=1}^N {y_i \over 4} = {1 \over 4}
}
Therefore there exists an $i$ such that $x_i^2 y_i \geq 1 / (4N)$ as required.
\end{proof}

\begin{lemma}\label{lem_w}
Let $a, b > 2$ and $x := 4a (\log ab)^2$. Then $x \geq a \log bx$.
\end{lemma}

\begin{lemma}\label{lem_tech2}
Let $\alpha_j := \ceil{\alpha^j}$ where $\alpha := \constalpha$. Then
$\sum_{j=1}^\infty {\alpha_j^{-1}} \leq 4\sqrt{N}$.
\end{lemma}

\begin{proof}
We have ${1 \over \alpha_j} \leq \left({1 \over \alpha}\right)^j < 1$.
Therefore by the geometric series,
\eq{
\sum_{j=1}^\infty {1 \over \alpha_j} \leq {1 \over 1 - {1 \over \alpha}} \equiv {1 \over 1 - {4\sqrt N - 1 \over 4\sqrt N}} = 4\sqrt{N}
}
as required.
\end{proof}

\section{Constants}
\begin{tabular}{p{1.5cm} p{6cm}}
$d$ & $\constd$ \\[0.2cm]
$\epsilon_\kappa$ & $2^{\kappa - 2} \epsilon$ \\[0.2cm]
$\gmax$ & $\constgmax$ \\[0.2cm]
$\gmaxkappa$ & $\constgmaxkappa$ \\[0.2cm]
$\emax$ & $\constemax$ \\[0.2cm]
$\emaxkappa$ & $\constemaxkappa$  \\[0.2cm]
$\alpha$ & $\constalpha$ \\[0.2cm]
$\delta_1$ & $\constdeltaone$ \\[0.2cm]
$|\K|$ & $\constK$
\end{tabular}

\section{Table of Notation}

\begin{tabular}{p{1.5cm} p{13cm}}
$N$ & number of candidate models \\[0.2cm]
$\epsilon$ & required accuracy \\[0.2cm]
$\delta$ & probability that an algorithm makes more mistakes than its sample-complexity \\[0.2cm]
$t$ & time-step \\[0.2cm]
$h_t$ & history at time-step $t$ \\[0.2cm]
$V^\pi_\mu(h)$ & value of policy $\pi$ in environment $\mu$ given history $h$ \\[0.2cm]
$d$ & effective horizon \\[0.2cm]
$\mu$ & true environment \\[0.2cm]
$\nu$ & an environment \\[0.2cm]
$\bnu,\unu$ & models achieving upper and lower bounds on the value of the exploration policy \\[0.2cm]
$\gamma$ & discount factor Satisfies $\gamma \in (0, 1)$ \\[0.2cm]
$\emaxkappa$ & high probability bound on the number of $\kappa$-exploration phases \\[0.2cm]
$\emax$ & high probability bound on the number of exploration phases \\[0.2cm]
$E_\infty$ & number of exploration phases \\[0.2cm]
$E_\infty(\nu,\kappa)$ & number of $\nu$-exploration phases \\[0.2cm]
$E_t(\nu,\kappa)$ & number of $(\nu,\kappa)$-exploration phases at time-step $t$ \\[0.2cm]
$F_t(\nu,\kappa)$ & number of effective $(\nu,\kappa)$-exploration phases at time-step $t$ \\[0.2cm]
$X(\nu,\kappa)_i$ & $i$th test statistic for $(\nu,\kappa)$ pair 
\end{tabular}

\end{document}